\documentclass[nohyperref]{article}

\usepackage{microtype}
\usepackage{graphicx}
\usepackage{subfigure}
\usepackage{booktabs} 
\usepackage{xspace}

\usepackage{hyperref}

\usepackage[accepted]{icml2022}

\usepackage{amsmath}
\usepackage{amssymb}
\usepackage{mathtools}
\usepackage{amsthm}

\usepackage{dsfont}

\usepackage[capitalize,noabbrev]{cleveref}

\theoremstyle{plain}
\newtheorem{theorem}{Theorem}[section]

\newtheorem{axiom}{Axiom}[section]

\newtheorem{proposition}[theorem]{Proposition}

\theoremstyle{definition}
\newtheorem{example}[theorem]{Example}
\newtheorem{definition}[theorem]{Definition}

\theoremstyle{remark}

\usepackage[disable,textsize=tiny]{todonotes}

\newcommand{\dfn}[1]{\textit{#1}}

\newcommand{\defeq}{\ensuremath{\stackrel{\mathrm{def}}{=}}}
\newcommand{\st}{\;|\;}

\renewcommand{\Re}{\mathbb{R}}
\newcommand{\Nat}{\mathbb{N}}

\renewcommand{\O}{\mathcal{O}}
\renewcommand{\L}{\mathcal{L}}
\renewcommand{\S}{\mathcal{S}}
\renewcommand{\P}{\mathbb{P}}
\newcommand{\T}{\mathbb{T}}
\newcommand{\A}{\mathcal{A}}
\newcommand{\W}{\mathcal{W}}
\newcommand{\D}{\mathcal{D}}
\newcommand{\E}{\mathbb{E}}

\newcommand{\VNMS}{VNM$\star$\xspace}
\newcommand{\VNMD}{VNM$\dag$\xspace}

\newcommand{\pge}{\succsim}
\newcommand{\pg}{\succ}
\newcommand{\peq}{\approx}
\newcommand{\ple}{\precsim}

\newcommand{\eg}[0]{\emph{e.g.},~}
\newcommand{\ie}[0]{\emph{i.e.},~}

\usepackage[framemethod=TikZ]{mdframed}
\mdfdefinestyle{MyFrame}{
    linecolor=black,
    outerlinewidth=.3pt,
    roundcorner=5pt,
    innertopmargin=\baselineskip,
    innerbottommargin=\baselineskip,
    innerrightmargin=5pt,
    innerleftmargin=5pt,
    backgroundcolor=black!0!white}

\mdfdefinestyle{MyFrame2}{
    linecolor=white,
    outerlinewidth=1pt,
    roundcorner=5pt,
    innertopmargin=\baselineskip,
    innerbottommargin=\baselineskip,
    innerrightmargin=5pt,
    innerleftmargin=5pt,
    backgroundcolor=black!3!white}

\newenvironment{myframe}[1]
  {\mdfsetup{
    frametitle={\colorbox{white}{\space#1\space}},
    innertopmargin=5pt,
    frametitleaboveskip=-\ht\strutbox,
    frametitlealignment=\center,
    linecolor=black,
    outerlinewidth=.3pt,
    roundcorner=5pt,
    innertopmargin=\baselineskip,
    innerbottommargin=\baselineskip,
    innerrightmargin=5pt,
    innerleftmargin=5pt,
    backgroundcolor=black!0!white
    }
  \begin{mdframed}
  }
  {\end{mdframed}}

\surroundwithmdframed[
   topline=false,
   rightline=false,
   bottomline=false,
   leftmargin=5pt,
   skipabove=\medskipamount,
   skipbelow=\medskipamount
]{example}

\icmltitlerunning{Utility Theory for Sequential Decision Making}

\begin{document}

\twocolumn[
\icmltitle{Utility Theory for Sequential Decision Making}

\begin{icmlauthorlist}
\icmlauthor{Mehran Shakerinava}{to,goo}
\icmlauthor{Siamak Ravanbakhsh}{to,goo}
\end{icmlauthorlist}

\icmlaffiliation{to}{School of Computer Science, McGill University, Montreal, Canada}
\icmlaffiliation{goo}{Mila - Quebec AI Institute}

\icmlcorrespondingauthor{Mehran Shakerinava}{mehran.shakerinava@mila.quebec}

\icmlkeywords{Machine Learning, ICML}

\vskip 0.3in
]

\printAffiliationsAndNotice{}  

\begin{abstract}
The von Neumann-Morgenstern (VNM) utility theorem shows that under certain axioms of rationality, decision-making is reduced to maximizing the expectation of some utility function. We extend these axioms to increasingly structured sequential decision making settings and identify the structure of the corresponding utility functions. In particular, we show that memoryless preferences lead to a utility in the form of a per transition reward and multiplicative factor on the future return. This result motivates a generalization of Markov Decision Processes (MDPs) with this structure on the agent's returns, which we call Affine-Reward MDPs. A stronger constraint on preferences is needed to recover the commonly used cumulative sum of scalar rewards in MDPs. A yet stronger constraint simplifies the utility function for goal-seeking agents in the form of a difference in some function of states that we call potential functions. Our necessary and sufficient conditions demystify the reward hypothesis that underlies the design of rational agents in reinforcement learning by adding an axiom to the VNM rationality axioms and motivates new directions for AI research involving sequential decision making.
\end{abstract}

\section{Introduction}
\label{sec:intro}

Utility theory is a proposal for rational behavior when faced with risky outcomes. 
Maximization of expected utility was originally hypothesized by Bernoulli~\citep{bernoulli_original, bernoulli_translation} as a solution to the St. Petersburg paradox, in which diminishing marginal utility explains human risk aversion in a game of chance. This hypothesis was later grounded by von Neumann and Morgenstern (VNM), such that as long as one's preferences satisfied certain rationality axioms, one's behavior could be explained as maximization of some utility function in expectation~\citep{vonneumann1947theory}. 
This paper aims to extend utility theory to sequential decision making. Our primary motivation is to ground what is known as the \emph{reward hypothesis} in reinforcement learning (RL): ``That all of what we mean by goals and purposes can be well thought of as maximization of the expected value of the cumulative sum of a received scalar signal (called reward).''~\citep{sutton2018reinforcement}.
While the connection between the reward in RL and the concept of utility in game theory 
has not gone unnoticed (\eg \citet{jaquette1976utility}), the adequacy of cumulative sum of scalar rewards still remains a hypothesis~\citep{sutton2018reinforcement}.

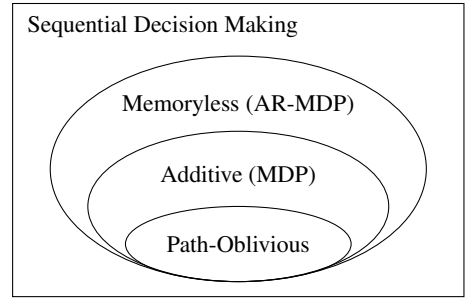
\begin{figure}
\centering
\begin{tikzpicture}[font=\small]

    \draw[fill=red!70,fill opacity=0] (-3,-1.7) rectangle (3,2.2);
    
    \draw[fill=red!70,fill opacity=0] (0,0) ellipse (2.5cm and 1.5cm);
    \draw[fill=red!70,fill opacity=0] (0,-0.5) ellipse (2cm and 1cm);
    \draw[fill=red!70,fill opacity=0] (0,-1) ellipse (1.5cm and 0.5cm);

    \node at (-1,1.9) {Sequential Decision Making};
    \node at (0,0.85) {Memoryless (AR-MDP)};
    \node at (0,-0.1) {Additive (MDP)};
    \node at (0,-1) {Path-Oblivious};
\end{tikzpicture}
\caption{
We axiomatize preference relations in sequential decision making and identify the corresponding structure for the utility function. The memorylessness, additivity, and path-obliviousness axioms impose successively stronger constraints on preferences over trajectories in sequential decision making. As a result, the corresponding utility functions become more structured. In particular, the additivity axiom leads to the additive reward structure of the MDP. Memoryless preferences lead to a more general AR-MDP, while path-oblivious preferences that only consider start and end states lead to a utility that is expressed as change in a potential function of states.
}
\label{fig:example}
\end{figure}

We identify necessary and sufficient conditions in sequential decision making that guarantee the existence of scalar reward signals, whose cumulative sum can represent any set of preferences over trajectories. This condition is presented as a single additional axiom to those of VNM, which itself justifies maximization of expected utility. Moreover, we place this particular structure of the utility function, \ie cumulative sum of scalar rewards, among several other possibilities that are based on more or less stringent assumptions.
In particular, we show that applying a memorylessness property to preference relations leads to a more general setup than Markov Decision Processes (MDPs) in which the utility is not the cumulative sum of rewards, but rather each step in addition to its additive contribution to the reward has a multiplicative effect on future rewards. We call this generalization of an MDP, an \emph{Affine-Reward MDP (AR-MDP)}. 

In the following, first, we review the foundations of utility theory in \cref{sec:background}. Then, in \cref{sec:sequential-decision-making} through \cref{sec:goal}, we consider a successively more structured setting for sequential decision making through additional axioms that constrain preferences over trajectories. We identify a one-to-one correspondence with an increasingly structured utility function in each case. These utility functions include a utility in which each state-transition has an associated additive reward and a multiplicative factor on future rewards as a result of a Markovian assumption (\cref{sec:markovian-sequential-decision-making}), commonly used cumulative rewards (\cref{sec:mdp}), and utilities that are expressed as a difference of potentials in the case of goal-seeking agents (\cref{sec:goal}). In \cref{sec:partial-preferences} we consider preference relations that have not been fully specified and how, in some cases, it is possible to uniquely complete them. \cref{sec:works} briefly reviews some of the relevant literature on utility theory. Finally, \cref{sec:discussion} discusses some implications of our results and considers possible extensions and exciting future directions that are motivated by these findings. 

\section{Background}
\label{sec:background}

We are interested in studying ``rational" decision-making. A pre-requisite to decision-making is the ability to compare outcomes through a \emph{preference relation}. But what does it mean for a preference relation to be rational? We will specify the meaning of rationality by introducing \emph{axioms} that agree with our intuitive understanding of rationality.

Let $\O$ denote the set of possible outcomes\footnote{We will assume that $\O$ is countable.}. We write $x \pge y$ if we prefer outcome $x$ to outcome $y$ (denoted as $x \succ y$) or if we are indifferent between the two outcomes (denoted as $x \peq y$). When $x \peq y$ we say that $x$ and $y$ are \dfn{equivalent}.

\begin{axiom}[Completeness]\label{ax:ordinal-complete}
For all $x, y \in \O$, $x \pge y$ or $y \pge x$, \ie any pair of outcomes are comparable.
\end{axiom}

\begin{axiom}[Transitivity]\label{ax:ordinal-transitive}
For all $x, y, z \in \O$, if $x \pge y$ and $y \pge z$, then $x \pge z$.
\end{axiom}

Note that completeness implies reflexivity (\ie for all $x \in \O$, $x \pge x$). Such preference relations are also sometimes known as a \emph{total preorder}.

Let us now consider uncertain outcomes. When a choice has an uncertain outcome, there will be a probability $p(x)$ of obtaining each outcome $x$. We will refer to such a choice as a \dfn{lottery}\footnote{A lottery is identified by a probability distribution $p$ over the outcome space $\O$, so in this sense, lotteries can be thought of as probability distributions.}. We may also consider \dfn{compound lotteries}, \ie lotteries of lotteries. Such lotteries can always be simplified into a single non-compound lottery. We will denote a general lottery of $n$ items as $\sum_{i=1}^n p(x_i)x_i$, where each $x_i$ is an outcome or a lottery. You can avoid confusing this notation with an expectation by noting that outcomes can't be multiplied and added.

\begin{example}
The lottery

\begin{equation*}
L = \frac{1}{2}x + \frac{1}{3}y + \frac{1}{6}M
\end{equation*}

means there is a $\frac{1}{2}$ chance of obtaining outcome $x$, a $\frac{1}{3}$ chance of obtaining outcome $y$, and a $\frac{1}{6}$ chance of obtaining an outcome according to another lottery $M$.
\end{example}

The framework introduced thus far does not allow us to make optimal decisions when faced with \emph{uncertain outcomes}. Suppose, for example, that there are three outcomes: $x \pg y \pg z$. When faced with a choice between $y$ and $\frac{1}{2} x + \frac{1}{2} z$ we are not able to say which choice is better. The reason is that there is a fundamental issue with a preference over outcomes: It does not specify how much we \emph{value} each outcome. For example, in this case, we know that $x$ is preferred to $y$, but how much more is it preferred?
To solve this issue, we must move to \emph{preferences over lotteries}. We will restate our current axioms to apply to lotteries and add two more axioms. Let $\L$ be the set of all lotteries of outcomes.

\begin{myframe}{VNM axioms}
\begin{axiom}[Completeness]\label{ax:completeness}
For all $L, M \in \L$, $L \pge M$ or $M \pge L$, \ie any pair of lotteries are comparable.
\end{axiom}

\begin{axiom}[Transitivity]\label{ax:transitivity}
For all $L, M, N \in \L$, if $L \pge M$ and $M \pge N$, then $L \pge N$.
\end{axiom}

\begin{axiom}[Continuity]\label{ax:continuity}
For all lotteries $L \pge M \pge N$, there exists $p \in [0, 1]$ such that $pL + (1-p)N \peq M$.
\end{axiom}

\begin{axiom}[Independence]\label{ax:independence}
For all $L, M, N \in \L$ and for all $p \in [0, 1]$,
\begin{equation}
L \pge M \iff (1 - p)L + p N \pge (1-p)M + p N.    
\end{equation}
\end{axiom}
\end{myframe}

The continuity axiom essentially states that, as the probabilities of a lottery vary, our valuation of the lottery changes smoothly.

The independence axiom can be understood by considering each compound lottery as a two-stage process. In the first stage, a coin with a probability $p$ of landing heads is tossed for picking a lottery, and in the second stage, an outcome is sampled from that lottery. If the coin lands heads, we get $N$ in both cases, but if it lands tails, we can get either $L$ or $M$. We may imagine that the coin has already been tossed, but the result has not yet been revealed to us. If the coin landed heads, our decision will not matter, but if it landed tails, we prefer to choose $L$. It seems reasonable, therefore, that $(1 - p)L + p N$ should be preferred to $(1-p)M + p N$. The same reasoning also applies in the opposite direction.
A corollary of independence is that, in a compound lottery, we can replace one lottery with another equivalent lottery, and the compound lotteries will be equivalent.

These four axioms are known as the \dfn{von Neumann-Morgenstern (VNM) axioms} and a preference relation over lotteries that satisfies these axioms is called \dfn{VNM-rational}.

It would be convenient if we could assign a value to each lottery such that comparing these values produces the same result as the preference relation.
Such a function, if it exists, can be thought of as an \emph{encoding} of its corresponding preference relation. This concept is captured by utility functions.

\begin{definition}[Utility function]
A utility function is a function $u: \L \to \mathbb{R}$, such that for all $L, M \in \L$,
\begin{equation}
L \pge M \iff u(L) \geq u(M).
\end{equation}
\end{definition}

Interestingly, VNM-rationality induces a utility function that is \emph{unique} up to positive affine transformation such that the utility of any lottery is equal to the \emph{expected} utility of its outcomes. This fact is formalized below.

\begin{mdframed}[style=MyFrame2]
\begin{theorem}[Von Neumann-Morgenstern utility theorem]\label{thm:vnm-util}
A preference relation satisfies the VNM axioms, if and only if it can be represented by a utility function such that for all lotteries with probability $p$,
\begin{equation}\label{eq:linear}
u \left ( \sum_{x \in \O} p(x)x \right ) = \sum_{x \in \O} p(x)u(x).
\end{equation}
Furthermore, this utility function is unique up to positive affine transformation.
\end{theorem} 
\end{mdframed}

\begin{proof}
See the appendix of \citet{vonneumann1953theory} for the original proof or \citet{maschler_solan_zamir_2013} for a simplified proof.
\end{proof}

Utility functions that satisfy \cref{eq:linear} are called \dfn{linear utility functions}.
A utility function that represents a VNM-rational preference relation is called a \dfn{VNM-utility}. The VNM utility theorem justifies the objective of maximizing expected utility. However, one must make sure that the utility that is being maximized is indeed a VNM-utility. 

\begin{example}
Consider the set of outcomes $\mathcal{O} = \{ \square, \circ, \triangle, \star \}$ along with a VNM-rational preference relation $\succsim$ on its set of lotteries $\mathcal{L}$. Suppose that $\square \succ \circ \succ \triangle \succ \star$.

We aim to construct a linear utility function $u$ on $\mathcal{L}$. We start by setting $u(\square) = 1$ and $u(\star) = 0$. By continuity, there exists some $p$ such that $p\ \square + (1-p)\ \star \approx \circ$, so we set $u(\circ) = p$. There also exists some $q < p$ such that $q\ \square + (1-q)\ \star \approx \triangle$, so we set $u(\triangle) = q$. Finally, we set the utility of any lottery to the expected utility of its outcomes. The constructed utility function is thus linear and one can show that it matches our preferences. Freedom in picking $u(\square)$ and $u(\star)$, as long as $u(\square) > u(\star)$, is what makes the utility function free up to positive affine transformation.

Proofs of the VNM utility theorem show the existence of a linear utility function by constructing it in the same way as this example.
\end{example}

\section{Extension to Sequential Decision Making}
\label{sec:sequential-decision-making}

We will now extend classical utility theory to sequential decision making. In this setting, an outcome will no longer depend on a single decision but on a sequence of decisions. Our model of sequential decision making consists of an agent that is interacting with a world. The world is modeled as a (countable) set of states $\S$. At each time-step $t \in \Nat$ the agent finds itself in state $s_t \in \S$ and must choose an action $a_t$ from a set $\A$ of actions; some of these actions may be illegal in state $s_t$.

We will assume that the result of an action depends only on the action and the current state (\ie Markov property). The result of an action, if legal, is to stochastically transition to a state and possibly terminate the interaction. The transition probabilities are given by $\P: \S \times \A \to \D(\S) \cup \{0\}$, where $\D(\S)$ is the space of probability distributions over $\S$ and $0$ indicates that the state-action pair is illegal. The termination probabilities are given by $\T: \S \times \A \times \S \to [0, 1]$. The tuple $(\S, \A, \P, \T)$ will be called a \dfn{Controlled Markov Process (CMP)}.

We define a \dfn{transition} to be a triplet $(s, a, s') \in \S \times \A \times \S$, with the interpretation that the agent chooses action $a$ in state $s$ and transitions to state $s'$. A \dfn{trajectory} of length $n$ is a sequence of transitions $\langle (s_i, a_i, s'_i) \rangle_{i \in \{1, ..., n\}}$ where $s'_i = s_{i+1}$ for all $i \in \{1, ..., n - 1\}$. For each state $s$, there is an empty trajectory $\epsilon_s$ that starts and ends in state $s$. We will refer to all of these empty trajectories collectively as \dfn{the empty trajectory} and denote it with $\epsilon$. The start and end state of $\epsilon$ will be clear from context.
We will use $\mathcal{T}$ as a short-hand for the set of transitions $\S \times \A \times \S$ and we will let $\mathcal{T}^*$ denote the set of finite trajectories.

In sequential decision making, the set of outcomes will be the set of finite trajectories of a CMP\footnote{The only role of a CMP in the case of utility theory is to specify which trajectories have a non-zero probability and get included in the set of outcomes. The probabilities themselves and the termination probabilities don't matter.}, \ie $\O = \mathcal{T}^*$, and preferences will be defined over \emph{lotteries of trajectories}. 

The VNM utility theorem may be applied in this setting, without any additional assumptions, to assign utilities to all finite trajectories of a CMP. However, each trajectory would then be considered as an entirely independent entity, and none of the structure of the CMP would be incorporated into the utility function. Furthermore, an optimal decision-making function would, generally, have to be a function of the agent's entire past trajectory. We will refer to a decision-making function of this general form as a \dfn{policy}. More specifically, a policy is a function $\pi: \mathcal{T}^* \to \D(\A)$.

\begin{example}
We will use the CMP seen in \cref{fig:example} as our running example. For simplicity, we will assume that preferences do not depend on the actions in a trajectory, and thus, trajectories can be written as a sequence of visited states. We also assume that $\langle s_0, \hat{s}_1, s_2 \rangle \pge \langle s_0, s_1, s_2 \rangle$ and $\langle s_2, \hat{s}_3 \rangle \pge \langle s_2, s_3 \rangle$.

With only the VNM axioms, all orderings of the trajectories are possible. For example, we could have

\begin{align*}
\langle s_0, \hat{s}_1, s_2, \hat{s}_3 \rangle &\pge \langle s_0, \hat{s}_1, s_2, s_3 \rangle, \ \text{and}\\
\langle s_0, s_1, s_2, s_3 \rangle &\pge \langle s_0, s_1, s_2, \hat{s}_3 \rangle.
\end{align*}

Then, an agent that started from state $s_0$ and is now in state $s_2$ will have to consider its past trajectory to decide if it prefers to go to state $\hat{s}_3$ or state $s_3$. 
\end{example}

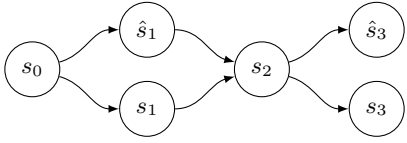
\begin{figure}
\centering
\begin{tikzpicture}[auto,node distance=8mm,>=latex,font=\small]
    \tikzstyle{round}=[draw=black,circle,minimum size=0.73cm]
    \node[round] (s0) {$s_0$};
    \node[round,above right=0mm and 10mm of s0] (s1) {$\hat{s}_1$};
    \node[round,below right=0mm and 10mm of s0] (s2) {$s_1$};
    \node[round,below right=0mm and 10mm of s1] (s3) {$s_2$};
    \node[round,above right=0mm and 10mm of s3] (s4) {$\hat{s}_3$};
    \node[round,below right=0mm and 10mm of s3] (s5) {$s_3$};
    \draw[->] (s0) to [out=+15,in=180] (s1);
    \draw[->] (s0) to [out=-15,in=180] (s2);
    \draw[->] (s1) to [out=0,in=165] (s3);
    \draw[->] (s2) to [out=0,in=195] (s3);
    \draw[->] (s3) to [out=+15,in=180] (s4);
    \draw[->] (s3) to [out=-15,in=180] (s5);
\end{tikzpicture}
\caption{The CMP that we use as our running example. We ignore actions and write trajectories as a sequence of visited states. We also assume $\langle s_0, \hat{s}_1, s_2 \rangle \pge \langle s_0, s_1, s_2 \rangle$ and $\langle s_2, \hat{s}_3 \rangle \pge \langle s_2, s_3 \rangle$.}
\label{fig:example}
\end{figure}

It seems reasonable to assume that in a Markovian world (\eg a CMP), where all the information relevant for predicting the future is contained in the present state, all the information necessary to \emph{compare} future events should also be contained in the present state. To incorporate such a Markovian assumption, we will need to constrain preference relations through an additional axiom, introduced in the next section. 

\section{Memoryless Sequential Decision Making}
\label{sec:markovian-sequential-decision-making}

Before introducing the axiom, we need to define a concatenation operator ``$\cdot$'' for trajectories. It has the property that it distributes over the addition operation used to create lotteries, that is, for all states $s$, trajectories $\tau$ and $\tau'$ that end in state $s$, and lotteries\footnote{By a lottery that starts from state $s$, we mean a lottery of trajectories that start from state $s$.} $L$ that start from state $s$, 
\begin{equation}
(p\tau + (1-p) \tau') \cdot L = p (\tau \cdot L) + (1-p) (\tau' \cdot L).
\end{equation}

We are now ready to augment the VNM axioms with the following axiom which asserts that one should be able to ignore the past trajectory when comparing future lotteries.

\begin{mdframed}[style=MyFrame]
\begin{axiom}[Memorylessness]\label{ax:memorylessness}
For all states $s$, trajectories $\tau$ that end in state $s$, and lotteries $L$ and $M$ that start from state $s$,
\begin{equation}\label{eq:memorylessness}
\tau \cdot L \pge \tau \cdot M \iff L \pge M.
\end{equation}
\end{axiom}
\end{mdframed}

We will use the shorthand, \dfn{\VNMS axioms}, for the VNM axioms along with the memorylessness axiom, and similarly for other terms that include VNM. 

\begin{example}
Consider the running example of \cref{fig:example} and the recall our   preferences: $\langle s_0, \hat{s}_1, s_2 \rangle \pge \langle s_0, s_1, s_2 \rangle$ and $\langle s_2, \hat{s}_3 \rangle \pge \langle s_2, s_3 \rangle$.
Then, memorylessness implies

\begin{align*}
\langle s_0, \hat{s}_1, s_2, \hat{s}_3 \rangle &\pge \langle s_0, \hat{s}_1, s_2, s_3 \rangle, \ \text{and}\\
\langle s_0, s_1, s_2, \hat{s}_3 \rangle &\pge \langle s_0, s_1, s_2, s_3 \rangle,
\end{align*}

however, it does not specify how $\langle s_0, \hat{s}_1, s_2, \hat{s}_3 \rangle$ compares to $\langle s_0, s_1, s_2, s_3 \rangle$. Under the \VNMS axioms, any preference is possible for these two trajectories.

As a result of the preferences implied by memorylessness, an agent that started from state $s_0$ and is now in state $s_2$, does not need to consider its past trajectory to make an optimal decision.
\end{example}

With the addition of the memorylessness axiom, the utility function is more structured, and we can encode the preferences much more succinctly by specifying only \emph{two numbers per transition}, instead of one number per trajectory. These two numbers are the reward and reward multiplier. This fact is formalized and proven below.

\begin{mdframed}[style=MyFrame2]
\begin{theorem}[\VNMS utility theorem]\label{thm:vnms-utility}
A preference relation over lotteries of finite trajectories of a CMP satisfies the \VNMS axioms, if and only if
there exists rewards $r: \mathcal{T} \to \Re$ and reward multipliers $m: \mathcal{T} \to \Re^+$, such that for all transitions $t$ and follow-up trajectories $\tau$, 
\begin{align}
    u(\epsilon) &\defeq 0\\
    u(t \cdot \tau) &\defeq r(t) + m(t)u(\tau),\label{eq:rec-affine}
\end{align}
is a linear utility function representing the given preference relation.

Moreover, $r$ is unique up to positive scaling and $m$ is unique, except for transitions that can only be followed by trajectories that are equivalent to $\epsilon$. For such transitions, $m$ can be chosen arbitrarily.
\end{theorem}
\end{mdframed}

\begin{proof}
We first assume that the \VNMS axioms hold and show how to construct $r$ and $m$. The VNM axioms tell us that there exists a linear utility function that is unique up to positive affine transformation. We pick one such utility function $u$ such that $u(\epsilon) = 0$, where $\epsilon$ is the empty trajectory. This $u$ is unique up to positive scaling.

Let $t = (s, a, s')$ be an arbitrary transition and Let $L$ and $M$ be any two lotteries that start from state $s'$. The memorylessness axiom tells us that preferences over lotteries that start from state $s'$ are the same as preferences over lotteries of trajectories that start with transition $t$.

We may conclude, by the VNM utility theorem, that for all lotteries $L$ that start from state $s'$, $u(t \cdot L)$ must be a positive affine transformation of $u(L)$. The parameters of this positive affine transformation give us $r(t)$ and $m(t)$. If $s'$ can only be followed by trajectories that are equivalent to $\epsilon$, then $u(L) = 0$ and $m(t)$ can be chosen arbitrarily, otherwise, $m(t)$ is unique because scaling $u$ does not change $m$. Scaling $u$, scales $r$ correspondingly.

We now show that $u$ satisfies the \VNMS axioms. Because $u$ is a linear utility function, the VNM utility theorem tells us that it satisfies the VNM axioms. The memorylessness axiom is also satisfied because prepending a trajectory to a lottery results in a positive affine transformation of its utility according to repeated application of \cref{eq:rec-affine} and this transformation preserves ordering.
\end{proof}

\cref{thm:vnms-utility} motivates the definition of what we call Affine-Reward MDP (AR-MDP).
An \dfn{AR-MDP} is a CMP combined with a reward function $r: \mathcal{T} \to \Re$ that assigns a scalar to each transition and a reward multiplier function $m:\mathcal{T} \to \Re^+$. The \emph{return} or utility associated with a trajectory $t \cdot \tau$ is recursively defined as
\begin{align}
    u(t \cdot \tau) = r(t) + m(t) u(\tau),
\end{align}
where $u(\epsilon) = 0$.

When given a CMP and a preference relation $\succsim$, our ultimate goal is to find an \dfn{optimal policy}, \ie a policy that achieves maximum expected utility. The memorylessness axiom, along with the Markov property, guarantee that there exists an optimal policy that depends only on the current state. We will refer to such a policy as a \dfn{memoryless policy}. More specifically, a memoryless policy is a function $\pi: \S \to \D(\A)$. Before formalizing this statement, we will need to briefly discuss how to compare policies.

We first define the utility of an \emph{infinite} trajectory $\tau$ to be $\lim_{T \to \infty} u(\tau_{:T})$, where $\tau_{:T}$ is the trajectory consisting of the first $T$ transitions of $\tau$. The \dfn{value} of a (general) policy $\pi$ in state $s$ is then defined as $v^\pi(s) \defeq \E_\pi[u(\boldsymbol{\tau}) \st s]$, where $\boldsymbol{\tau}$ is a random variable denoting the infinite trajectory taken by an agent starting from state $s$ and following policy $\pi$. Policy $\pi_1$ is preferred to policy $\pi_2$ in state $s$ if $v^{\pi_1}(s) > v^{\pi_2}(s)$. A complication is that the limit might not exist and so, we may not be able to compare some policies. To avoid this problem, we assume in the following proposition that the limit exists.

\begin{proposition}\label{thm:policy-opt}
Given a CMP $\W$ and a \VNMS preference relation over lotteries of all finite trajectories of $\W$ such that $v^\pi(s)$ exists for all policies $\pi$ and all states $s$, there exists an optimal policy that is memoryless.
\end{proposition}

\begin{proof}
Let $\pi^\star_s$ be an optimal policy starting from state $s$. Consider an agent that has arrived in state $s$ via trajectory $\tau$. The goal of the agent is to find $\arg \max_\pi \E_\pi[u(\tau \cdot \boldsymbol{\tau'}) \st \tau]$ where $\boldsymbol{\tau'}$ is a random variable representing the future trajectory. Using the \VNMS Theorem one can see that the objective is equivalent to
\begin{align*}
    & \arg \max_\pi \E_\pi[u(\tau) + m(\tau) u(\boldsymbol{\tau'}) \st \tau]\\
   =& \arg \max_\pi \E_\pi[u(\boldsymbol{\tau'}) \st \tau] & (m(\tau) > 0)\\
   =& \arg \max_\pi \E_\pi[u(\boldsymbol{\tau'}) \st s] & (\mathrm{Markov\ property})\\
   =& \arg \max_\pi v^\pi(s)\\
   =& \pi^\star_s,
\end{align*}
where $m(\tau) = \prod_{t \in \tau} m(t)$. Therefore, the optimal action for the agent is given by $\pi^\star_s(\epsilon_s)$. This observation is true for all states, therefore, $\pi^\star(s) \defeq \pi^\star_s(\epsilon_s)$ is a memoryless policy that is simultaneously optimal for all states.
\end{proof}

\section{An Axiom for Markov Decision Processes}\label{sec:mdp}

An \dfn{MDP} is a CMP combined with a reward function $r: \mathcal{T} \to \Re$ that assigns a scalar to each transition. In an MDP, the utility of a trajectory $\tau$ is evaluated as $u(\tau) = \sum_{t \in \tau} r(t)$.\footnote{This definition of MDPs slightly differs from the usual definition which does not include termination probabilities $\T$. We also do not include a discount factor $\gamma$, but it is easy to simulate one by modifying the termination probabilities as $\T_\mathrm{new} = 1 - \gamma(1 - \T_\mathrm{old})$.}

Since the \VNMS axioms incorporate a Markovian property, one might expect them to correspond to MDPs, but as we saw, this is not the case, and instead, we arrived at AR-MDPs which are more general than MDPs: If we set $m(t) = 1$ for all transitions $t$, then we arrive at MDPs. An important question, therefore, arises: \emph{What additional assumptions do MDPs make about preferences?} The additional axiom corresponding to an MDP is what we call additivity.

\begin{mdframed}[style=MyFrame]
\begin{axiom}[Additivity]\label{additivity}
For all states $s$, trajectories $\tau_1$ and $\tau_2$ that end in state $s$, lotteries $L$ and $M$ that start from state $s$, lotteries $N$ and $K$, and $p \in [0, 1]$,

\begin{align}
     & p(\tau_1 \cdot L) + (1 - p) N \pge p (\tau_1 \cdot M) + (1 - p) K \notag\\
\iff & p(\tau_2 \cdot L) + (1 - p) N \pge p (\tau_2 \cdot M) + (1 - p) K. 
\end{align}

\end{axiom}
\end{mdframed}

This axiom is similar to memorylessness in the sense that it requires that changing the initial trajectory of two lotteries should maintain preference relations. The difference with memorylessness is that, here, we are allowed to change the initial trajectory of equal-probability sub-lotteries, which makes this axiom \emph{stronger than memorylessness}. To arrive at memorylessness, let $\tau_2 = \epsilon_s$ and $N = K$, and use independence (\cref{ax:independence}) to remove $N$ and $K$ from the comparison.

The additivity axiom is somewhat difficult to interpret. One of the ways to better understand it is through its implications: if parts of a trajectory are known and fixed and some parts are unknown and must be optimized, then additivity says that each part can be optimized independently and the known parts of the trajectory can be entirely ignored. It might be easy to check if such an assumption holds for a given task.\footnote{Note that this implication is not a sufficient condition for additivity to hold.}

We will use the shorthand, \dfn{VNM+ axioms}, for the VNM axioms along with the additivity axiom, and similarly for other terms that include VNM.

\begin{mdframed}[style=MyFrame2]
\begin{theorem}[VNM+ utility theorem]\label{thm:vnm+-utility}
A preference relation over lotteries of finite trajectories of a CMP satisfies the VNM+ axioms, if and only if
there exists a reward function $r: \mathcal{T} \to \Re$, such that for all transitions $t$ and follow-up trajectories $\tau$,
\begin{align}
    u(\epsilon) &\defeq 0\\
    u(t \cdot \tau) &\defeq r(t) + u(\tau),\label{eq:rec-add}
\end{align}
is a linear utility function representing the given preference relation.
Moreover, $r$ is unique up to positive scaling.
\end{theorem}
\end{mdframed}

\begin{proof}
We first assume that the VNM+ axioms hold and we show that it is possible to obtain utilities as in \cref{eq:rec-add}. Since the VNM+ axioms imply the \VNMS axioms, the \VNMS utility theorem lets us specify utilities as in \cref{eq:rec-affine} via functions $r$ and $m$. We will show that when the additivity axiom holds, we can set $m(t) = 1$ for all transitions $t$.

For transitions $t$ that can only be followed by trajectories that are equivalent to $\epsilon$, $m(t)$ can be chosen arbitrarily, so we can set it to $1$.
Let $t = (s, a, s')$ be an arbitrary transition among the remaining transitions. We will show that $m(t) = 1$. Let $\tau$ be a trajectory following $t$ which is not equivalent to $\epsilon$.
\begin{align*}
 &\frac{1}{2}\tau + \frac{1}{2}\epsilon_{s'} \peq \frac{1}{2}\epsilon_{s'} + \frac{1}{2}\tau \\
 \implies& \frac{1}{2}(t \cdot \tau) + \frac{1}{2}\epsilon_{s'} \peq \frac{1}{2}t + \frac{1}{2}\tau  & \text{(additivity)} \\
 \implies& \frac{1}{2}(r(t) +  m(t)u(\tau)) = \frac{1}{2}u(\tau) + \frac{1}{2}r(t) \\
 \implies& m(t)u(\tau) = u(\tau)\\
 \implies& m(t) = 1 & (u(\tau) \neq 0)\\
\end{align*}

We now show that $u$ satisfies the VNM+ axioms. If we let $m(t) = 1$ for all transitions $t$, we see that by the \VNMS utility theorem, $u$ satisfies the VNM axioms. It remains to show that $u$ satisfies the additivity axiom. Because utilities are additive (\cref{eq:rec-add}), changing an initial trajectory adds the utility difference of the old and new trajectories to the utility of the lottery, thus, ordering is preserved.
\end{proof}

\begin{example}
Consider the running example of \cref{fig:example} and recall our partial preference assumptions:
$\langle s_0, \hat{s}_1, s_2 \rangle \pge \langle s_0, s_1, s_2 \rangle$ and $\langle s_2, \hat{s}_3 \rangle \pge \langle s_2, s_3 \rangle$.
With additivity axioms our preference between
 $\langle s_0, \hat{s}_1, s_2, \hat{s}_3 \rangle$ and
 $\langle s_0, s_1, s_2, s_3 \rangle$ is now constrained. In contrast, memorylessness does not constrain this preference. 

To see this, note that one of the implications of \cref{thm:vnm+-utility} is that, for all trajectories $\tau_1, \tau_2, \hat{\tau}_1$, and $\hat{\tau}_2$, such that $\tau_2$ follows $\tau_1$ and $\hat{\tau}_2$ follows $\hat{\tau}_1$,
\begin{equation}\label{eq:dominance}
    \hat{\tau}_1 \pge \tau_1 \ \text{and}\  \hat{\tau}_2 \pge \tau_2 \implies \hat{\tau}_1 \cdot \hat{\tau_2} \pge \tau_1 \cdot \tau_2.
\end{equation}

Now, our preference assumptions along with \cref{eq:dominance} imply that $\langle s_0, \hat{s}_1, s_2, \hat{s}_3 \rangle \pge \langle s_0, s_1, s_2, s_3 \rangle$.
\end{example}

\section{Goal-Seeking Sequential Decision Making}
\label{sec:goal}

In many settings the objective is to reach the best possible state, \ie the means of achieving something do not matter, all that matters is the final result. Some examples are chess, freestyle swimming, 
and tennis. Not all sports fall into this category. In sports such as gymnastics, figure skating, and diving, \emph{how} the task is performed (\ie the entire trajectory) matters. We will introduce an axiom to account for such settings.

\begin{mdframed}[style=MyFrame]
\begin{axiom}[Path-obliviousness]\label{ax:oblivious}
For all $p \in [0, 1]$, states $s$ and $\tilde{s}$, lotteries $L,M,\tilde{L},\tilde{M},N$ and $K$, such that $L$ and $M$ start from state $s$, $\tilde{L}$ and $\tilde{M}$ start from state $\tilde{s}$, and the final-state distribution of $\tilde{L}$ and $\tilde{M}$ is the same as that of $L$ and $M$, respectively,
\begin{align}
      p L + (1 - p) N  & \pge p M + (1 - p) K \notag \\
\iff  p \tilde{L} + (1 - p) N & \pge p \tilde{M} + (1 - p) K.
\end{align}
\end{axiom}
\end{mdframed}

This axiom resembles additivity in the sense that changing the initial trajectory of equal-probability sub-lotteries preserves ordering. Here, however, we are allowed to change the starting state and entire trajectories as long as the final-state distribution stays the same. \emph{Path-obliviousness is stronger than additivity}. To see this, note that letting $L = \tau_1 \cdot \hat{L}$, $\tilde{L} = \tau_2 \cdot \hat{L}$, $M = \tau_1 \cdot \hat{M}$, and $\tilde{M} = \tau_2 \cdot \hat{M}$ recovers additivity.

Additionally, path-obliviousness implies that two trajectories that start from the same state and end in the same state are equivalent. Let $\tau_1$ and $\tau_2$ be two such trajectories. Then let $L = \tilde{M} = \tau_1$, $M = \tilde{L} = \tau_2$, and $N = K$, and use independence to remove $N$ and $K$ to obtain $\tau_1 \pge \tau_2 \iff \tau_2 \pge \tau_1$, which implies $\tau_1 \peq \tau_2$.

We will use the shorthand, \dfn{\VNMD axioms}, for the VNM axioms along with the path-obliviousness axiom, and similarly for other terms that include VNM.

\begin{example}
Consider the running example of \cref{fig:example}. The inclusion of the path-obliviousness axiom will constrain the preferences even further, \eg $\langle s_0, \hat{s}_1, s_2 \rangle \peq \langle s_0, s_1, s_2 \rangle$ and $\langle s_0, \hat{s}_1, s_2, s_3 \rangle \peq \langle s_0, s_1, s_2, s_3 \rangle$. As a result, assuming that the CMP does not terminate after one step, the action of the agent in state $s_0$ does not matter because the agent will eventually end up in $s_2$ and all trajectories that go from $s_0$ to $s_2$ have the same utilities and utilities are additive.

\end{example}

For the theorem that we are about to introduce, we will make the simplifying assumption that there exists a state $s_0$ from which all of the states of the CMP are reachable.

\begin{mdframed}[style=MyFrame2]
\begin{theorem}[\VNMD utility theorem]\label{thm:vnmd-utility}
A preference relations over lotteries of finite trajectories of a CMP $\W$, in which all states are reachable from some state $s_0$, satisfies the \VNMD axioms, if and only if
there exists a function $\phi: \S \to \Re$ such that for all states $s$ and $s'$, and trajectories $\tau$ starting from state $s$ and ending in state $s'$,
\begin{align}
    u(\epsilon) &\defeq 0\\
    u(\tau) &\defeq \phi(s') - \phi(s),\label{eq:util-pot}
\end{align}
is a linear utility function representing the given preference relation. Moreover, $\phi$, called the \emph{potential} function, is unique up to positive affine transformation.
\end{theorem}
\end{mdframed}

\begin{proof}
We first assume that the \VNMD axioms hold. Since path-obliviousness implies additivity, we may invoke the VNM+ utility theorem to obtain additive utilities that are unique up to positive scaling. We will now construct the function $\phi$. We set $\phi(s_0)$ to an arbitrary value and for any other state $s$, we pick an arbitrary trajectory $\tau$ that goes from state $s_0$ to state $s$ and set $\phi(s) = \phi(s_0) + u(\tau)$. The choice of trajectory does not matter because of the path-obliviousness axiom; thus, $\phi(s)$ is well-defined. In this way, we have constructed the potential function $\phi$. Because we are free in choosing $\phi(s_0)$ and the positive scaling of the utilities, $\phi$ is unique up to positive affine transformation.

Next, we show that utilities can be obtained from this potential function $\phi$. Let $\tau$ be an arbitrary trajectory starting from state $s$ and ending in state $s'$ and let $\tau_0$ be any trajectory that goes from state $s_0$ to state $s$.
\begin{align*}
\phi(s) - \phi(s_0) + u(\tau) &= u(\tau_0) + u(\tau) \\
 &= u(\tau_0 \cdot \tau) \\
 &= \phi(s') - \phi(s_0) \\
 \implies u(\tau) = \phi(s') - \phi(s)&
\end{align*}

We now assume that utilities can be obtained from a potential function $\phi$. It is easy to see that utilities are additive in this case. Therefore, the VNM+ axioms must hold. Also, path-obliviousness holds, since changing the starting state of a trajectory from state $s$ to state $s'$ changes the utility by $\phi(s') - \phi(s)$ which preserves ordering.
\end{proof}

\section{Related Works}
\label{sec:works}

Until the mid-twentieth century, utility theory relied on preference structures that did not explicitly incorporate uncertainty or probability. Specifying assumptions for characterizing rational behavior under uncertainty began, in a sense, with the classical paper of \citet{bernoulli_original} and was later developed and formalized in large part due to \citet{ramsey1926}, \citet{de1937prevision}, \citet{vonneumann1947theory} and \citet{savage1954foundations}. These works sparked renewed interest in the role of uncertainty in preference structures.

We have focused on the utility theory developed in~\citep{vonneumann1947theory} as the basis for our work. Their work focuses mainly on a game-theoretical setting, as opposed to a general sequential decision making setting, and preferences were applied to entire plays of a game to show that it is possible to assign utilities such that optimal behavior corresponds to maximizing expected utility. Such an approach has become standard in game theory; see, for example~\citep{maschler_solan_zamir_2013}. 
In games that have a sequential nature, it is common to assume that the game eventually terminates and that the final state determines the outcome of the game. In such cases, preferences are applied to the terminal states. 
Expected utility theory would then, for example, justify the use of an algorithm such as {Expectiminimax}~\citep{michie1966game, russell_artificial_1995} 

in a non-deterministic two-player zero-sum game.
Such an approach does not apply to a general sequential decision making setting because, in many scenarios, the entire trajectory should be evaluated, not only the final state, and in some scenarios, the interaction might never terminate.

There are many works that have considered extensions of utility theory to the setting where the set of outcomes has the structure of a product space~\citep{debreu1959topological, fishburn_utility_1970, keeney_decisions_1976}. They show that under certain conditions, there exists an additive utility function. Note that the set of trajectories is not a product space because many combinations of transitions are invalid. That is why we needed stronger axioms than those introduced in these earlier works.

Also, a condition called stationarity has been proposed, which is somewhat similar to our memorylessness axiom~\citep{koopmans1960stationary}. In the product space setting, one can view the product space as a time series, then stationarity states that changing the initial segment of two outcomes should not affect their comparison.

To our knowledge, there is only one work that focuses on extending utility theory to a general sequential decision making setting, namely~\citep{pitis2019rethinking}. 
They add two axioms and one assumption to the VNM axioms to obtain the equivalent of our \VNMS Theorem, whereas we only add a single axiom. They also consider the outcome space to be the set of state and policy pairs which is a large continuous high-dimensional space whereas we use the set of finite trajectories which is a countable space. In these regards, we believe that our approach is simpler. We also go beyond Affine-Reward MDPs and provide theorems for (additive) MDPs and goal-directed agents.

Another relevant attempt at a generalization of MDPs is through Constrained MDPs, which maximize a certain utility while satisfying constraints on other utilities~\citep{altman1999constrained}. See \citet{csaba} for implications for the reward hypothesis.

\section{Discussion}
\label{sec:discussion}

The reward hypothesis refers to ``goals and purposes'', but what exactly does that mean? If the goal is to achieve some desired behavior, specifically, a desired deterministic memoryless policy $\pi^\star$, then the hypothesis is true, because we can define the reward function as $r(s, a, s') = +1$ if $a = \pi^\star(s)$ and $-1$ otherwise.

In this work, we view rational preferences as a very precise specification of goals and purposes. Not only do they specify what behavior is optimal and what behavior is sub-optimal (in a 0-1 fashion), but they also allow us to compare any two behaviors, \ie we can say \emph{how good} a behavior is.

We have shown that, in this interpretation of the reward hypothesis, in the case of Markovian preferences, expected cumulative reward may not be enough to encode preferences and that an additional reward multiplier signal is also required. Only when our preferences satisfy the additivity axiom, in addition to being VNM-rational, does the additive reward suffice. This result can also be of importance to practitioners of inverse reinforcement learning as capturing an agent's preferences by a reward function may not produce adequate results unless we are sure that the agent's preferences adhere to the VNM+ axioms.

We may also think of alternative ways of defining goals and purposes; see, for example, \citet{abel2021expressivity}. These alternatives can usually be converted into preferences in a non-unique way because they are, essentially, incomplete preference specifications. 
If we can convert a set of goals into preference relations that satisfy the VNM+ axioms, then we can successfully express those goals via reward functions. Sometimes, a goal might not even be convertible to a Markovian preference relation. In this case, rewards and reward multipliers are not enough, and a memoryless policy may not be able to produce optimal behavior. One solution to consider, in this case, is modifying the state-space to include more information from the past.

We will now briefly mention some exciting avenues for future work. It would be interesting to study the possibility of extending these theorems to the setting of continuous state/action-space or continuous time (see \cref{sec:countability} for a proposal).

Another implication of our findings is the potential importance of AR-MDPs. Identification of important real-world scenarios where the \VNMS axioms hold but VNM+ may not hold, and the design of efficient learning algorithms for AR-MDPs merits investigation.

\section*{Acknowledgements}
We thank the anonymous reviewers for their valuable comments. This project is in part supported by the CIFAR AI chairs program and NSERC Discovery Grant. 

\bibliography{main}
\bibliographystyle{icml2022}

\newpage
\appendix
\onecolumn

\section{Countability of the Set of Outcomes}
\label{sec:countability}

The VNM utility theorem, in its original form, is only applicable when the set of outcomes is countable. We explain here why the set of finite trajectories is countable.

\begin{proposition}
The set of finite trajectories of a CMP, whose states and actions are countable, is countable.
\end{proposition}

\begin{proof}
Because the set of states and actions of the CMP are countable, they are isomorphic to $\mathbb{Z}_{\geq 1}$. Consequently, there is a one-to-one mapping of non-empty finite trajectories to non-empty finite sequences of positive integers. If we consider the \emph{continued fraction representation} of real numbers, there is a bijection between $\mathbb{Q}_{\geq 1}$ and non-empty finite sequences of positive integers. Therefore, the set of finite trajectories fits inside $\mathbb{Q}_{\geq 1}$, which is a countable set.
\end{proof}

To apply the VNM utility theorem to uncountable sets of outcomes, an additional axiom, known as the sure-thing principle~\citep[p.~77]{savage1954foundations}, is required.

\begin{axiom}[Sure-thing principle]

For all lotteries $L$ with probability measure $p$, lotteries $M$, and sets $\mathcal{X}$ such that $p(\mathcal{X}) = 1$,

\begin{equation}
    \forall x \in \mathcal{X}: x \pge M \implies L \pge M \qquad \mathrm{and} \qquad \forall x \in \mathcal{X}: x \ple M \implies L \ple M.
\end{equation}

\end{axiom}

Incorporating this axiom is one way that would allow us to include infinite trajectories as part of the set of outcomes or consider an uncountable set of states/actions or continuous time.

\section{Partially Specified Preferences}
\label{sec:partial-preferences}

Since preference relations are constrained, a subset of them may be enough to recover all preference relations. We identify one such interesting subset in this section. In particular, we will assume that only preferences over lotteries that start from a fixed initial state $s_0$ are known. Note that we are not assuming that the preference relation is incomplete, only that some of the preferences are not revealed to us. 

\begin{proposition}\label{thm:partial-preferences}
If a preference relation over lotteries of finite trajectories of a CMP satisfies the VNM+ axioms, knowing only preferences over lotteries that start from a fixed initial state $s_0$ uniquely determines all preferences over lotteries of trajectories reachable from state $s_0$.
\end{proposition}
 
\begin{proof}
From the known preferences, we can construct a utility function for trajectories starting from state $s_0$. Now, consider an arbitrary trajectory $\tau$ that starts in state $s$ (which is reachable from $s_0$) and ends in state $s'$. Let $\tau'$ be a trajectory starting from state $s_0$ and ending in state $s$. Then, because VNM+-utilities are additive, the utility of trajectory $\tau$ can be obtained as $u(\tau' \cdot \tau) - u(\tau')$. These utilities let us compare all lotteries of trajectories that are reachable from state $s_0$, and thus, all preferences are now determined.
\end{proof}

\section{Alternative Axioms}

In this section, we will explore a few alternative axioms.

Instead of the additivity axiom, one may employ memorylessness along with the following axiom.

\begin{axiom}
For all $\tau_1$, $\tau_2$ that end in state $s$, and all $\tau_3$, $\tau_4$ that start from state $s$,

$$
\frac{1}{2}\tau_1 \cdot \tau_3 + \frac{1}{2}\tau_2 \cdot \tau_4 \approx \frac{1}{2}\tau_1 \cdot \tau_4 + \frac{1}{2}\tau_2 \cdot \tau_3.
$$
\end{axiom}

It is easy to see that this axiom can replace additivity in the proof of the VNM+ theorem. This axiom has been mentioned in \citet{meyer1976} in the context of extending utility theory to real-valued time series.

It is possible to replace the path-obliviousness axiom with the additivity axiom and an axiom that says any two trajectories with the same start and end states are equivalent.

\end{document}